\documentclass[12pt]{amsart}
\usepackage{graphicx, overpic}
\usepackage[below]{placeins}
\usepackage[colorlinks=true, linkcolor=blue, citecolor=blue]{hyperref}
\usepackage[]{algorithm2e}

\usepackage[T1]{fontenc}

\usepackage{amsmath,amsthm,amscd,amssymb,eucal}

\usepackage{enumerate, amsfonts, latexsym, color, url}
\usepackage{epstopdf}

\usepackage{pinlabel}

\makeatletter
\newsavebox{\@brx}
\newcommand{\llangle}[1][]{\savebox{\@brx}{\(\m@th{#1\langle}\)}%
  \mathopen{\copy\@brx\kern-0.5\wd\@brx\usebox{\@brx}}}
\newcommand{\rrangle}[1][]{\savebox{\@brx}{\(\m@th{#1\rangle}\)}%
  \mathclose{\copy\@brx\kern-0.5\wd\@brx\usebox{\@brx}}}
\makeatother

\begin{document}

\newtheorem{theorem}{Theorem}[section]
\newtheorem{lemma}[theorem]{Lemma}
\newtheorem{proposition}[theorem]{Proposition}
\newtheorem{corollary}[theorem]{Corollary}
\newtheorem{conjecture}[theorem]{Conjecture}
\newtheorem{question}[theorem]{Question}
\newtheorem{problem}[theorem]{Problem}
\newtheorem*{claim}{Claim}
\newtheorem*{criterion}{Criterion}

\theoremstyle{definition}
\newtheorem{definition}[theorem]{Definition}
\newtheorem{construction}[theorem]{Construction}
\newtheorem{notation}[theorem]{Notation}
\newtheorem{object}[theorem]{Object}
\newtheorem{operation}[theorem]{Operation}

\theoremstyle{remark}
\newtheorem{remark}[theorem]{Remark}
\newtheorem{example}[theorem]{Example}

\numberwithin{equation}{subsection}

\newcommand\id{\textnormal{id}}

\newcommand\M{\mathcal M}
\newcommand\N{\mathbb N}
\newcommand\Z{\mathbb Z}
\newcommand\R{\mathbb R}
\newcommand\C{\mathbb C}
\newcommand\D{\mathbb D}
\newcommand\HH{\mathbb H}
\newcommand\CP{\mathbb{CP}}
\newcommand\V{\mathcal V}
\renewcommand\SS{\mathcal S}
\newcommand\RR{\mathcal R}
\newcommand\PL{\mathcal{PL}}
\newcommand\SL{\textnormal{SL}}
\newcommand\PSL{\textnormal{PSL}}
\newcommand\PMF{\mathcal{PMF}}
\newcommand\Teich{\textnormal{Teich}}
\newcommand\tr{\textnormal{trace}}

\title{Triangulating PL functions and the existence of efficient ReLU DNNs}

\author{Danny Calegari}
\address{University of Chicago \\ Chicago, Ill 60637 USA}
\email{dannyc@uchicago.edu}
\date{\today}

\begin{abstract}
We show that every piecewise linear function $f:\R^d \to \R$ with compact support a polyhedron $P$ has a
representation as a sum of {\em simplex functions}. Such representations arise from degree
1 triangulations of the relative homology class (in $\R^{d+1}$) bounded by $P$ and
the graph of $f$, and give a short elementary proof of the existence of efficient universal ReLU
neural networks that simultaneously compute all such functions $f$ of bounded complexity.
\end{abstract}

\maketitle
\setcounter{tocdepth}{1}

\section{Introduction}

The purpose of this note is to give a short proof of a normal form theorem for piecewise-linear
functions $\R^d \to \R$ of compact support, which gives rise to a fixed architecture ReLU neural network
capable of computing all such functions (of fixed complexity). The proof is direct and constructive and gives
more or less optimal constants, and is (from the perspective of a topologist) much more 
elementary than existing proofs in the literature. 

We emphasize that this is a new proof of a {\em well-known theorem} with {\em many 
well-known proofs in the literature} (see e.g.\/ \cite{Petersen_Zech} Chapter 5 especially 
Theorem 5.14, and e.g.\/ 
\cite{Arora_Basu_Mianjy_Mukherjee,He_Li_Xu_Zheng,Longo_Opschoor_Disch_Schwab_Zech,Ovchinnikov,Petersen_Zech,Tarela_Martinez,Wang_Sun})
The proofs in these references all rely on the efficient representation of PL functions as sums of functions
of an explicit sort, and our proof is no exception. The novelty of our method (such as it is)
is that our representation is tied directly to the combinatorial (simplicial) structure of the
graph of the function in question, so that firstly its meaning is more transparent and it may be computed
very easily, and secondly one may obtain {\em lower bounds} on the complexity
of the representation by geometric methods (explicitly: hyperbolic geometry). Thus we hope to
suggest an analogy to other well-known contexts in theoretical computer science where there is
a direct connection between hyperbolic volume and computational complexity --- 
e.g.\/ the Sleator--Tarjan--Thurston theory of rotations of binary trees \cite{Sleator_Tarjan_Thurston},
or the Cohn--Kenyon--Propp theory of entropy of dimer tilings \cite{Cohn_Kenyon_Propp}.

\section{An approximation theorem}

Fix a dimension $d$ and an integer $n$. Let $\PL(d,n)$ be the class of piecewise-linear
functions from $\R^d$ to $\R$ with compact support which are linear on a triangular mesh
of at most $n$ simplices (and zero away from these simplices).

We first define the class of {\em simplex functions}. The meaning of these functions is very
simple: take a simplex $\Delta$ in $\R^{d+1}$ made of rubber, and `drop' it vertically onto $\R^d$. It
will deform into a pyramid whose base is a rather simple polyhedron in $\R^d$, and whose height at every
point $p$ is the length of the intersection of the vertical line through $p$ with $\Delta$.
A {\em simplex function} is the function whose graph is the top of this pyramid.

\begin{definition}[Simplex function]
Say that an $d+1$-simplex in $\R^{d+1}$ is {\em nondegenerate} if it is full dimensional in $\R^{d+1}$,
and if no face of $\Delta$ is vertical (i.e.\/ the image of every face of $\Delta$
is a full dimensional $d$-simplex in $\R^d$). 
The {\em simplex function} $\tau(\Delta):\R^d \to \R$ is the element
of $\PL(d,n)$ whose value at a point $x\in \R^d$ is the length of the segment 
$\ell_x \cap \Delta$, where $\ell_x$ is the vertical line in $\R^{d+1}$ whose first $d$
coordinates give $x$.
\end{definition}

\begin{lemma}[Max-min normal form for simplex function]\label{lemma:simplex_function}
Let $\tau(\Delta):\R^d \to \R$ be a simplex function. Then there are at most $d^2$ {\em linear} functions
$g_1,\cdots,g_{d^2}:\R^d \to \R$ so that
$$\tau(\Delta) = \max \left( 0, \min_i g_i\right)$$
Consequently there is a fixed ReLU neural network $T$ with 
\begin{enumerate}
\item{depth at most $2 \log_2(d)+2$;}
\item{width at most $d^2 C_1$; and}
\item{size at most $d^2 C_2$}
\end{enumerate}
so that for every simplex function $\tau(\Delta)$ there is some assignment of weights and biases to $T$ that
computes $\tau$.
\end{lemma}
\begin{proof}
It is well-known (and elementary) how to construct a simple depth $1$ ReLU neural network 
that can compute $\max$ or $\min$ of two real-valued inputs. Thus the lemma is proved if we can exhibit a
description of $\tau$ of the desired form.

Let $P \subset \R^d$ be the image of $\Delta$ under orthogonal projection. By nondegeneracy,
the image of each face of $\Delta$ is a full dimensional simplex. Thus $P$ is the convex
hull of one of these simplices $D$ and an additional point $q\in \R^d$ which is the image of
a vertex $\hat{q}$ of $\Delta$. The polyhedron $P$ therefore has at most $d^2$ faces, 
which are the projections of some of the codimension 2 faces of $\Delta$ (the correct upper
bound is quadratic in $d$; there is no point writing down a precise number). We claim that 
the graph of $\tau$ (in $\R^{d+1}$) is the top of a pyramid with base $P$ and apex a point $\hat{p}$
projecting to some $p$ in the interior of $P$.

This is an exercise in elementary Euclidean geometry, but it is easy enough to give
a sketch of a proof, which explains the meaning of $\hat{p}$ and $p$.

Extend the faces of $D$ to planes in $\R^d$ so that they form a hyperplane arrangement.
The point $q$ is in one of the chambers of this arrangement, dual to a face $F$ of $D$
which is the image of a face $\hat{F}$ of $\Delta$. Let $\hat{F}'$ be the face of $\Delta$
spanned by the vertices of $\Delta - \hat{F}$ and $\hat{q}$, and let $F'$ be the projection of
$\hat{F}'$. Then $F$ and $F'$ intersect transversely at a point $p \in P$, and the graph
of $\tau(\Delta)$ is a pyramid, obtained by coning $P$ to a point $\hat{p}$ in $\R^{d+1}$
mapping to $p$, and such that the last coordinate of $\hat{p}$ is $\tau(p)$.

In particular, if we triangulate $P$ by coning its faces to $p$ (which gives a decomposition
into at most $d^2$ simplices) then $\tau$ is linear on each of these simplices; and if
$g_i$ are the linear functions that agree with $\tau$ on these simplices, then 
$\tau = \min_i g_i$ on $P$. Since $\tau$ is $0$ outside $P$, we have $\tau = \max(0,(\min_i g_i))$
as claimed.
\end{proof}

\begin{remark}
It is easy to write a simplex function as a bounded linear sum of simplex functions each
of which is the max of the min of $d+\text{constant}$ linear functions, but explaining this
detail seems superfluous for a note whose virtue is its brevity.
\end{remark}

\begin{theorem}\label{decomposition_theorem}
Any $f\in \PL(d,n)$ may be written as the sum of at most $2n$ simplex functions. Hence 
there are small, easily estimated constants $C_0, C_1, C_2$ so that for any 
$d$ and $n$ there is a ReLU neural network $\Phi$ with 
\begin{enumerate}
\item{depth at most $2 \log_2(d)+C_0$;}
\item{width at most $d^2 n C_1$; and}
\item{size at most $d^2 n C_2$}
\end{enumerate}
so that for every $f \in \PL(d,n)$ there is some assignment of weights and biases to $\Phi$ that computes $f$.
\end{theorem}
\begin{proof}
Let's pick a function $f$ in $\PL(d,n)$. We may assume (by adding more simplices if necessary) that 
the support of $f$ is a polyhedral ball $C$ in $\R^d$ which we think of as the subspace of $\R^{d+1}$ with
last coordinate zero. Let $C' \subset \R^{d+1}$ be the graph of $f$ over $C$, thought of 
as an embedded polyhedral $d$-ball in $\R^{d+1}$. The union $S:=C \cup C'$ is an 
immersed $d-$sphere in $\R^{d+1}$. The mesh of simplices on which $f$ is linear defines a triangulation
of $C$ into $n$ simplices, which together with the obvious triangulation of $C'$ gives a
triangulation of $S$ with $2n$ simplices. 

Pick a generic point in $\R^{d+1}$ and cone the triangulation of $S$ to this point. This
gives a PL immersion of a triangulated $(d+1)$-ball $B$ into $\R^{d+1}$. The simplices of $B$
all map (by general position) to nondegenerate simplices $\Delta_i$ in $\R^{d+1}$; we give
each $\Delta_i$ a sign $\epsilon_i = \pm 1$ depending on whether the immersion is positively or
negatively oriented when restricted to $\Delta_i$. One way of expressing this in terms
familiar to algebraic topologists is to say that the $\Delta_i$ give a {\em degree 1 triangulation}
of the relative homology class bounded by $S$ (for an introduction to simplicial homology 
see e.g.\/ \cite{Hatcher}, Chapter 2).

By construction $f = \sum \epsilon_i \tau(\Delta_i)$. Let $T$ be a fixed DNN of depth $2\log_2(d)+2$
that can compute $\tau(\Delta)$ for any nondegenerate simplex $\Delta$ for some assignment of
weights. Then $\Phi$ may be built from $2n$ disjoint parallel copies of $T$, together with a single 
neuron at the end that sums the inputs with linear weights. Evidently a suitable assignment of 
weights and biases to $\Phi$ computes $f$.
\end{proof}

\section{Hyperbolic geometry}

For an introduction to the use of hyperbolic geometry to give lower bounds on the 
complexity of polyhedra, one can hardly do better than section 3 of \cite{Sleator_Tarjan_Thurston}
and we do not attempt to add anything to what they say.

If $f:\R^d \to \R$ is a PL function with support a compact polyhedron $P$, we 
say that a {\em simplicial representation of $f$} is a decomposition into simplex
functions $f = \sum \pm \tau(\Delta_i)$
where the $\Delta_i$ give a degree 1 triangulation of the relative homology class
bounded by the union of $P$ and its graph. We would like to bound the number of terms in
such a representation in terms of $f$. 

Let's introduce some notation and some terminology. We let $\HH^{d+1}$ denote hyperbolic space
of dimension $d+1$. If $S$ is a closed, oriented $d$-dimensional polyhedral manifold, we say an immersion
$g:S \to \HH^{d+1}$ is
{\em totally geodesic} if the image of each top dimensional polyhedral face of $S$ is a
totally geodesic $d$-dimensional polyhedron in $\HH^{d+1}$. 

\begin{theorem}
Let $f:\R^d \to \R$ be a PL function with support a compact polyhedron $P$. Let
$S$ be a polyhedral $d$-manifold obtained by gluing $P$ to the graph of $f|P$ (where 
the graph of $f|P$ is subdivided into the facets on which $f$ is linear). Then the number of
terms in any simplicial representation $f = \sum \pm \tau(\Delta_i)$ is bounded above
by $2n$, where $P$ may be decomposed into $n$ simplices on each of which $f$ is linear, 
and below by $V(f)/v_{d+1}$, where $V(f)$ is the supremum of the algebraic volumes in
$\HH^{d+1}$ enclosed by $g(S)$ over all totally geodesic immersions $g:S \to \HH^{d+1}$, and
$v_{d+1}$ is the volume of the regular ideal simplex in $\HH^{d+1}$.
\end{theorem}
\begin{proof}
The upper bound is given in the proof of Theorem~\ref{decomposition_theorem}. The 
lower bound follows by straightening: any simplicial representation of $f$ arises from a
degree one triangulation of the relative homology class bounded by $P$ and its graph
(by definition). Map each such simplex into $\HH^{d+1}$ compatibly with $g$, and then
straighten each simplex to be totally geodesic. Any totally geodesic simplex in
$\HH^{d+1}$ has volume at most $v_{d+1}$.
\end{proof}

\begin{conjecture}
Any identity $f = \sum \pm \tau(\Delta_i)$ arises from a simplicial representation of $f$
(i.e.\/ it arises from a degree 1 triangulation of a relative homology class).
\end{conjecture}

\end{document}